\documentclass[accepted]{uai2026} 
                        
\usepackage[american]{babel}

\usepackage{natbib} 
\usepackage{mathtools} 
\usepackage{booktabs} 
\usepackage{tikz} 

\usepackage{amssymb}
\usepackage{mathtools}
\usepackage{amsthm}
\usepackage{bm}
\usepackage{algorithm}
\usepackage{algpseudocode}
\usepackage{caption}
\usepackage{subcaption}

\usepackage{amsmath}
\theoremstyle{plain}
\newtheorem{theorem}{Theorem}[section]
\newtheorem{proposition}[theorem]{Proposition}

\theoremstyle{definition}
\newtheorem{definition}[theorem]{Definition}

\theoremstyle{remark}

\usepackage{wrapfig}

\title{Learning Lineage-guided Geodesics with Finsler Geometry}

\author[1,2]{Aaron Zweig\textsuperscript{*}}
\author[2,3]{Mingxuan Zhang\textsuperscript{*}}
\author[1,3,4]{David A. Knowles}
\author[2,3,4,5,6]{Elham Azizi}

\affil[1]{%
    New York Genome Center
}
\affil[2]{%
    Irving Institute for Cancer Dynamics, Columbia University
}
\affil[3]{%
    Department of Systems Biology, Columbia University
}
\affil[4]{%
    Department of Computer Science, Columbia University
}
\affil[5]{%
    Department of Biomedical Engineering, Columbia University
}
\affil[6]{%
    Data Science Institute, Columbia University
}
  
\begin{document}
\maketitle

\renewcommand\thefootnote{*}
\footnotetext{These authors contributed equally.}

\renewcommand\thefootnote{\arabic{footnote}}

\begin{abstract}
Trajectory inference investigates how to interpolate paths between observed timepoints of dynamical systems, such as temporally resolved population distributions, with the goal of inferring trajectories at unseen times and better understanding system dynamics.  Previous work has focused on continuous geometric priors, utilizing data-dependent spatial features to define a Riemannian metric.  In many applications, there exists discrete, directed prior knowledge over admissible transitions (e.g. lineage trees in developmental biology).  We introduce a Finsler metric that combines geometry with classification and incorporate both types of priors in trajectory inference, yielding improved performance on interpolation tasks in synthetic and real-world data.
\end{abstract}


\section{Introduction}\label{sec:intro}
In many scientific domains, data are observed at discrete timepoints while the underlying system evolves continuously in time. Trajectory inference aims to reconstruct continuous paths between empirical distributions, enabling interpolation to unseen times and analysis of system evolution.
Single-cell omics provides a prominent instance of this setting.  Although single-cell RNA sequencing is destructive, sampling at multiple timepoints enables a weakly temporal perspective on the evolution of cells, particularly in developmental settings where cell differentiation is similar among all healthy embryos.  However, the inability to sample continuously means a practitioner typically only has access to the cell distribution at a small number of timepoints.

In the context of trajectory inference, optimal transport (OT) is a popular modeling framework across applications because it follows the principle of minimal energy, where the inferred trajectories are optimal with respect to some underlying cost function.  In the context of Riemannian metrics, these trajectories are geodesics.  However, only in very special cases (Euclidean space or hyperspheres) are the geodesics available in closed form, and must otherwise be learned.

We are interested in settings where the underlying metric is data-dependent, and informed by prior domain knowledge.  Encouraging trajectories to stay near observed samples (e.g. cells)~\citep{kapusniak2024metric} provides a geometric prior, but one may also include a discrete prior in the form of directed constraints over admissible transitions (e.g. lineage information).  For example, if the literature on developmental cell states in a particular organism demonstrates that one state typically differentiates into another, we seek to enforce that knowledge in our metric to encourage biologically plausible trajectories. Crucially, such  priors arise when transitions are known to be structured, directed, or partially ordered (e.g., stage progressions, causal precedence, or permitted state changes), and cannot be captured by symmetric (Riemannian) distances alone.

In this work, we apply Finsler geometry to incorporate discrete, directed transition priors into the geometry used for trajectory inference. Single-cell developmental lineage trees serve as a practical example of this setting in real single-cell RNA sequencing data.  Specifically our contributions include (i) the definition of a Finsler metric conditioned on a directed adjacency matrix representing admissible transitions (lineage tree prior), (ii) formal proof that this metric induces well defined local geometry structure and enforces trajectories that agree with the directed  prior, and (iii) demonstration that this Finsler metric may be easily incorporated with geometric priors to improve the accuracy of the trajectories and interpolation of unseen timepoints.

\section{Preliminaries}

\begin{figure}[t]
     \centering
     \includegraphics[width=0.5\textwidth]{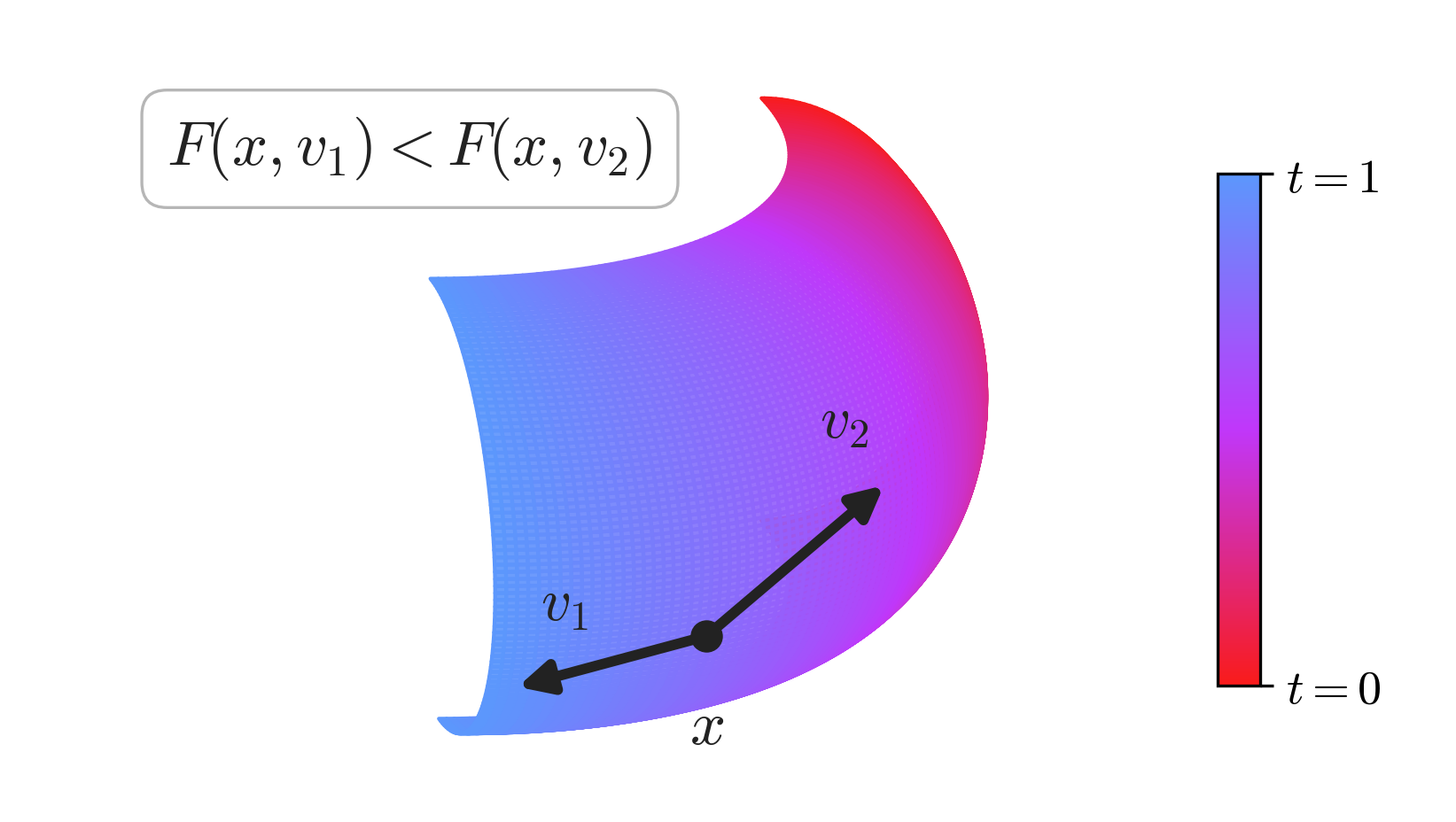}
    \caption{A visual representation of the Finsler metric as a penalty.  For temporal data evolving from $t=0$ (red) towards $t=1$ (blue), we want to define a local metric at $x$ such that, if $v_1$ agrees with the lineage prior and $v_2$ contradicts the prior, then $F(x, v_1) < F(x, v_2)$, i.e. the metric acts as a penalty on disagreeing with the classification signal.}
    \label{fig:manifold}
\end{figure}

\subsection{Notation}

We let $\mathbb{R}_+ = [0, \infty)$ denote non-negative reals, and denote the ReLU operator as $(\cdot)_+$.  We use $\{e_i\}_{i=1}^n$ to denote the standard basis on $\mathbb{R}^n$.  For a map $f: X \rightarrow \mathbb{R}^n$, we'll denote $f_i(x)$ as the $i$th index of $f(x)$.

\subsection{Riemannian Metric}
    We consider a smooth, orientable manifold $\mathcal{M}$ with metric tensor $g$ such that, for any $x \in \mathcal{M}$ and tangent vectors $u, v \in T_x \mathcal{M}$, we have the inner product $\langle u, v \rangle_{g(x)}$ and the associated norm $\|v\|_{g(x)} = \sqrt{\langle v, v \rangle_{g(x)}}$.

    In this setting, we have a well-defined notion of distance,
    \begin{align}
        d_g(x,y) = \inf_{\substack{{\gamma: [0,1] \rightarrow \mathcal{M}}\\ \gamma(0) = x, \gamma(1) = y}} \int_0^1 \|\dot{\gamma}(s)\|_{g(\gamma(s))} ds.
    \end{align}

A geodesic is a critical point of the energy functional,
\begin{align}
    E_g(\gamma)
    = \inf_{\substack{{\gamma: [0,1] \rightarrow \mathcal{M}}\\ \gamma(0) = x, \gamma(1) = y}} 
    \frac{1}{2}\int_0^1
    \|\dot{\gamma}(s)\|_{g(\gamma(s))}^2 \, ds.
\end{align}

Among curves with fixed endpoints, constant-speed minimizers of the energy coincide with minimizers of the distance.

\subsection{Finsler Metric}

The notion of a Riemannian metric may be generalized beyond a local inner product, to define a local asymmetric norm via the formalism of Finsler geometry~\citep{dahl2006brief}.  We have the following definitions:

\begin{definition}
    Given a vector space $X$, a function $p: X \rightarrow \mathbb{R}_+$ is an \emph{asymmetric norm} if it satisfies the conditions:
    \begin{enumerate}
        \item \textit{Triangle inequality}: For all $x, y \in X$, $p(x+y) \leq p(x) + p(y)$.
        \item \textit{Positive homogeneity}: For all $x \in X$, $\lambda \in \mathbb{R}_+$, we have $p(\lambda x) = \lambda p(x)$.
        \item \textit{Non-degeneracy}: $p(x) = 0$ if and only if $x = 0$.
    \end{enumerate}
\end{definition}

This definition is only distinct from a norm in that the homogeneity condition applies only to non-negative scalars, which enables $p(x) \neq p(-x)$.

\begin{definition}
    Given a manifold $\mathcal{M}$, we say $F: T\mathcal{M} \rightarrow \mathbb{R}_+$ is a \emph{Finsler metric} if for every $x \in \mathcal{M}$, $F(x, \cdot)$ defines an asymmetric norm on $T_x \mathcal{M}$.
\end{definition}

Crucially, the notions of geodesics, length and energy carry over to Finsler geometry~\citep{dahl2006brief}, with the corresponding definitions

\begin{align}
    d_F(x,y) &= \inf_{\substack{{\gamma: [0,1] \rightarrow \mathcal{M}}\\ \gamma(0) = x, \gamma(1) = y}} \int_0^1 F(\gamma(s), \dot{\gamma}(s)) ds, \\
    E_F(\gamma)
    &= \inf_{\substack{{\gamma: [0,1] \rightarrow \mathcal{M}}\\ \gamma(0) = x, \gamma(1) = y}} 
    \frac{1}{2}\int_0^1
    F(\gamma(s), \dot{\gamma}(s))^2 \, ds.
\end{align}

\section{Related Work}

\subsection{Trajectory Inference Methods}

Many popular methods for trajectory inference rely on ordinary differential equations (ODEs), either learned directly  with Neural ODEs~\citep{chen2018neural, tong2020trajectorynet} or implicitly through the regression loss of flow matching~\citep{lipman2023flow, tong2024improving}.  Trajectory methods specific to single-cell transcriptomics typically rely on similar tools such as optimal transport~\citep{schiebinger2019optimal} or noisy biological signals that approximate velocity~\citep{bergen2020generalizing}.

In particular, \citet{kapusniak2024metric} introduced a geometric prior perspective within flow matching with a learned Riemannian metric, among other methods for learning an underlying metric from data~\citep{aizenbud2025estimation, moon2019visualizing}.  Related generative models consider generalizing from ODEs to stochastic differential equations (SDEs) with parameterizations of a Schr\"{o}dinger Bridge between distributions~\citep{tong2024simulation}. These approaches assume symmetric local costs induced by Riemannian structure. In contrast, our setting incorporates additional discrete and directed priors that cannot be expressed through symmetric metrics.

\subsection{Cell Lineage Tree Priors}

Other works have used prior knowledge about the lineage tree to bias inference.  LineageOT~\citep{forrow2021lineageot} uses the lineage tree to infer direct ancestors in a maximum likelihood estimate, but requires an estimate of the most recent time that each cell changed state.  Moslin~\citep{lange2024mapping} uses a fused Gromov-Wasserstein loss that matches cell topology to lineage tree topology, but is limited by the cubic complexity of running Gromov-Wasserstein solvers.

\subsection{Riemannian and Finsler Metrics}

Many applications require calculation of geodesics under Riemannian geometry, often relying on numerical methods~\citep{kimmel1998computing,crane2017heat}.  Specific applications in machine learning use the same machinery, especially in the context of flow matching~\citep{de2025pullback,chen2024flow,kapusniak2024metric} or single-cell genomics~\citep{palma2025enforcing}.  The application of Finsler metrics is substantially less common. The main application we are aware of is extending metric learning algorithms to Finsler geometry in~\citet{dages2025finsler}, where a pairwise distance matrix is given. However, in our problem the metric itself is learned and geodesic distances must be approximated during training.

\section{Lineage-Informed Finsler Geometry}

\subsection{Lineage Tree Guided Finsler Metric}

We consider a single cell dataset with cells represented in as vectors in a reduced dimension $\mathbb{R}^n$ and cell-type labels from the discrete set $C$.  We assume access to a lineage tree with adjacency matrix $A \in \{0,1\}^{|C|\times|C|}$, such that $A_{ij} = 1$ if there is a direct edge from class $i$ to class $j$ (i.e., cell-type/state $i$ is known to be able to differentiate into state $j$), and assume $A_{ii} = 1$ indicating cells can remain in their current state.  Given this prior, we first seek to define a sensible Finsler metric.  Namely, we want to encourage geodesics in ambient space such that the induced trajectories of categorical cell-type distributions on the simplex $\Delta^{|C|}$ follow the tree.

To that end, we consider imposing a penalty on trajectories that instantaneously contradict the tree structure.  Let $f: \mathbb{R}^n \rightarrow \Delta^{|C|}$ denote a learned cell-type classifier that outputs a probability vector over the $|C|$ classes. Let $x_t$ parameterize a trajectory from $x_0$ to $x_1$.  We propose a penalty of the form,
\begin{align}\label{eq:penalty}
    \sum_{c \in C} f_c(x_t) \langle \partial_t f(x_t),(\mathbf{1} - A^T)e_c \rangle_+
\end{align}

where $\mathbf{1}$ is a matrix of all ones.  In other words, we take a sum over all classes weighted by the classifier-predicted probability at $x_t$, and for each class $c$ we penalize if the trajectory's velocity has positive correlation with the vector of \textit{illegal directions} from class $c$.  Note that we take the transpose of $A$ because it is acting as an asymmetric Markov Chain and so acts on vectors multiplied on the left.

In order for the Finsler metric to be defined independently of a given trajectory, we propose,
\begin{align}\label{eq:finsler_metric_definition}
    \tilde{F}(x,v) = \sum_{c \in C} f_c(x) \langle v, Jf(x)^T(\mathbf{1} - A^T)e_c \rangle_+
\end{align}
where $J$ is the Jacobian. 

In this notation, $\tilde{F}(x_t, \dot{x}_t)$ is equivalent to equation~\ref{eq:penalty} by the chain rule.  Because we want to combine geometric with classification priors, we propose to combine our metric with a given Riemannian metric $\|\cdot\|_{g}$ to give the final penalty, 
\begin{align}
    F(x, v) = \|v\|_{g(x)}\left(1 + \lambda \tilde{F}\left(x,\frac{v}{\|v\|}\right)\right)
\end{align}
where $\lambda \in \mathbb{R}_+$.

This formulation balances the two priors by treating the classification prior as a scaling factor on the geometric prior.  As we see below, as long as $g$ is a conformal Riemannian metric, this defines a valid Finsler metric.
\begin{proposition}
    Given that $\|v\|_{g(x)}$ is a non-trivial conformal Riemannian metric, $F$ defines a non-degenerate Finsler metric.
\end{proposition}
\begin{proof}
    By definition, we may write $\|v\|_{g(x)} = \|v\|G(x)$ for some strictly positive scalar map $G(x)$, and therefore the Finlser metric reduces to 
    \begin{align}
        F(x, v) = \|v\|_{g(x)} + \lambda G(x) \tilde{F}\left(x,v\right)
    \end{align}

    The space of Finsler norms is closed under addition and a non-degenerate Riemannian norm is strictly convex and hence a non-degenerate Finsler norm.  So it remains to show that $\lambda G(x)f_c(x_t) \langle v, Jf(x)^T(\mathbf{1} - A^T)e_c \rangle_+$ is a legitimate Finsler norm.  By the properties of the ReLU, this may be more succinctly written as $\max_{w \in C(x)} \langle v, w \rangle$ where $C(x)$ is the line segment from $0$ to $\lambda g(x)\langle f(x_t), e_c \rangle Jf(x)^T(\mathbf{1} - A^T)e_c$.  Finally, the dual norm of a convex set containing zero is automatically a Finsler norm~\citep{dahl2006brief}.
\end{proof}

\subsection{Parameterizing Finsler Geometry}

Given a Finsler metric $F$, we aim to efficiently approximate the distance $d_F(x,y)$ in order to enable practical OT along the metric.  This relates to the application of Finsler geometry-based multi-dimensional scaling, first explored in~\citet{dages2025finsler}, where pairwise distances are given.  We don't have oracle access to the distances between points, as our model is attempting to learn the geodesics and calculating the distances would require accurately simulating the integral over paths.  Instead, we derive a different loss to learn embeddings that preserve the Finsler distance, similar to a pullback metric for Riemannian geometry.

We consider a latent space of dimension $l$ and learnable maps $\phi, \psi: \mathcal{M} \rightarrow \mathbb{R}^l$ and a learnable vector $\beta \in \mathbb{R}^l$ and propose to approximate the distances as $d_F(x,y) \approx \|\phi(x) - \phi(y)\| + \langle \psi(x) - \psi(y), \beta\rangle_+$ where the second term introduces asymmetry in the distance.  Let $\gamma$ be the minimizing geodesic between $x$ and $y$, and let $x_t = \gamma(t)$.  Then if we consider $d_F(x_0, x_t)$ using the geodesic formulation, our approximation is of the form
\begin{align*}
    \int_0^t F(x_s, \dot{x}_s) ds &\approx \|\phi(x_0) - \phi(x_t)\| + \langle \psi(x_0) - \psi(x_t), \beta\rangle_+
\end{align*}

Applying $\lim_{t \rightarrow 0} \frac{1}{t}$ on both sides yields the first order Taylor approximation
\begin{align*}
    F(x_0, \dot{x}_0) &\approx \|J\phi(x_0)\dot{x}_0\| + \langle -J\psi(x_0)\dot{x}_0, \beta\rangle_+
\end{align*}

We next provide theoretical evidence of the soundness of our formulation. We first prove that our construction induces a local Finsler structure so we have valid length functional and path distance.

\begin{theorem}[Local Finsler structure]\label{thm:weak-finsler}
Let $\mathcal{M}\subset\mathbb{R}^n$ be a compact $C^1$ manifold, $\phi,\psi\in C^{1}(\mathcal{M};\mathbb{R}^{\ell})$, and $\beta\in\mathbb{R}^{\ell}$.
Define, for $(x,v)\in \mathcal{M}\times\mathbb{R}^n$,
\[
\hat{F}(x,v)\;:=\;\|J\phi(x)\,v\|_{2}\;+\;\big\langle -\,J\psi(x)\,v,\;\beta\big\rangle_{+},
\]
on the tangent bundle $T \mathcal{M}$. Fix $x_0\in \mathcal{M}$ and assume \emph{local non-degeneracy} of the symmetric part at $x_0$:
\begin{equation}\label{eq:nondeg}
\exists\,m_0>0\ \text{s.t.}\ \|J\phi(x_0)v\|_2\ \ge\ m_0\|v\|_2
\end{equation}
The following holds:
\begin{enumerate}[label=(\roman*)]
\item $\hat{F}(x, .)$ is an asymmetric norm. More specifically, the map $v\mapsto \hat{F}(x,v)$ is convex and positively $1$-homogeneous:
\[
\hat{F}(x,\alpha v)=\alpha \hat{F}(x,v)\quad\forall \alpha\ge 0,
\]
and $\hat{F}(x,v)>0$ for all $v\neq 0$.
\item We have lower bounds for the norm, i.e., $\exists m > 0$ such that:
\[
\hat{F}(x,v) \ge m\|v\|_2\, \quad \forall x \in \mathcal{M}, \forall v \in T\mathcal{M}
\]
\end{enumerate}
Consequently, $\hat{F}$ defines a Finsler structure on $\mathcal{M}$.
\end{theorem}
\begin{proof}
\textbf{Proof for (i).} In the first term $\|.\|_{2}$ is an Euclidean norm which is convex and positively 1-homogeneous. Hence, for $\alpha > 0$
\[
\|J\phi(x)\, (\alpha v)\|_{2} = \alpha \|J\phi(x)\ v\|_{2} > 0
\]
In the second term we have
\[
\big\langle \,-J\psi(x)\,v,\;\beta\big\rangle = - (J\psi(x)\,v)^T \beta = -v^T J\psi(x)^T \beta
\]
The map $v \mapsto -v^T J\psi(x)^T \beta$ is linear with respect to $v$. Hence for $\alpha > 0$ we have
\[
\big\langle \,-J\psi(x)\,(\alpha v),\;\beta\big\rangle = \alpha \big\langle \,-J\psi(x)\, v,\;\beta\big\rangle
\]
Additionally, ReLU operator is convex and positively 1-homogeneous on $\mathbb{R}$ for all $\alpha > 0$. Consequently, our second term is a composition of a convex function and a linear map. Hence the map $v \mapsto \big\langle \,-J\psi(x)\,v,\;\beta\big\rangle_+$ is convex and positive 1-homogeneous. Summing up the two terms, for every fixed $x$, the map
\[
v \mapsto \hat{F}(x, v)
\]
is also convex and positively 1-homogeneous. 

\textbf{Proof for (ii):}
Assume local non-degeneracy at $x_0$, we have some $m_0$ that
\[
\sigma_{min}(J\phi(x_0)) > m_0
\]
where $\sigma_{min}(.)$ is minimal singular value. Since $\phi \in C^1$, so the mapping $x \mapsto J\phi(x)$ is continuous. Hence  $\exists r > 0$ and a margin $\epsilon \in (0, m_0)$ such that for all $x \in \mathcal{M}$ and $v \in T\mathcal{M}$
\[
||J\phi(x) - J\phi(x_0)||_o < \epsilon
\]
where $||.||_o$ here represent spectral norm. Additionally, since singular values are 1-Lipschitz w.r.t. spectral norm, we have:
\begin{align*}
    \sigma_{min}(J\phi(x)) &\geq \sigma_{min}(J\phi(x_0)) - ||J\phi(x) - J\phi(x_0)||_o\\
    &> m_0 - \epsilon
\end{align*}

So set $m = m_0 - \epsilon$ we have
\[
\sigma_{min}(J\phi(x)) \geq m, \quad \forall x \in \mathcal{M}
\]

Hence we have the local lower bound
\[
||J\phi(x)v||_2 \geq \sigma_{min}(J\phi(x))||v||_2 \geq m||v||_2
\]
Since $\langle -J\psi(x)v, \beta\rangle_+ > 0$, we have 
\[
\hat{F}(x, v) \geq m||v||_2
\]

\end{proof}

We furthermore show that, assuming the data follows the same topology as the lineage tree and we have an ideal classifier, optimizing our loss function recovers the geodesic and produces non-contradicting pass.

\begin{theorem}[Geodesic recovery and non-contradicting path]
    Let $f$ be a differentiable classifier.  Consider a differentiable trajectory $x_t$ such that for any time $t$, $f(x_t)$ is always non-zero on at most two classes which may vary with $t$.  Furthermore, when supported on two classes $i$ and $j$ we assume there is an edge in the tree from $i$ to $j$ and $\langle \partial_t f(x_t), e_i \rangle < 0$.

    Then in the limit as $\lambda \rightarrow \infty$, $x_t$ is a geodesic in the Finsler metric.
\end{theorem}
\begin{proof}
    In this limit, the Riemann contribution to the Finsler metric becomes negligible, so we can focus on the asymmetric contribution.  

    Consider any time $t$, and suppose $f$ is supported on one class, i.e. $f(x_t) = e_i$.  Since each entry of $f$ is only supported on $[0,1]$, this is an extremum for each entry of $f$ and therefore $\partial_t f(x_t) = 0$ and $F(x_t, \dot{x}_t) = 0$.
    
    Alternatively, suppose at time $t$ the support is on two classes $i$ and $j$.  We can calculate,
    \begin{align*}
        F(x_t, \dot{x}_t) &= \sum_{c \in C} f_c(x_t) \langle \partial_t f(x_t),(\mathbf{1} - A^T)e_c \rangle_+ \\
        &= \sum_{c \in \{i, j\}} f_c(x_t) \langle \partial_t f(x_t),(\mathbf{1} - A^T)e_c \rangle_+
    \end{align*}

    Note that $\partial_t f(x_t)$ cannot have any support outside of indices $i$ and $j$, as there would otherwise be a point with support on three classes.  If $M = \mathbf{1} - A^T$, note that $M_{ii} = M_{ji} = M_{jj} = 0$ and $M_{ij} = 1$.  Hence, 
    \begin{align}
        F(x_t, \dot{x}_t) &= f_i(x_t) \langle \partial_t f(x_t),e_i \rangle_+\\
        &= 0
    \end{align}

    by the assumption that $\langle \partial_t f(x_t),e_i \rangle < 0$.  Hence, in the limit $x_t$ has zero energy along its trajectory and must be a geodesic.
\end{proof}

Hence, our Finsler metric serves as a sensible incorporation of the lineage prior.  

\subsection{Learning Finsler Geometry}

In order to learn the embedding and geodesics simultaneously, we define a loss to better approximate the true metric with the distance embedding and minimize the energy.  Namely, we can consider any distribution $\mu$ on the tangent bundle and the associated embedding loss
\begin{align*}
L_{\mathrm{emb}}&(\phi,\psi,\beta,\mu)
=\\ & \mathbb{E}_{x,v \sim \mu}\Big[
\big|\|J\phi(x)v\| + \langle -J\psi(x)v, \beta\rangle_+ - F(x,v)\big|
\Big].
\end{align*}

Given a network $\eta$ that parameterizes geodesics and a coupling between timepoints $\pi$, we have an energy loss:
\begin{align*}
    L_{geo}(\eta, \pi) &= E_{(x_0, x_1) \sim \pi}\left[F(x_t, \dot{x}_t)^2 \right] \\
    x_t &:= (1-t)x_0 + tx_1 + t(1-t)\eta(x_0, x_1, t)
\end{align*}

To make these losses concrete, we must choose sensible sampling distributions.  As in previous work~\citep{tong2024improving}, we choose $\pi$ to be the OT coupling between two observed timepoints, where we measure pairwise distances using the learned distance $d_{\hat{F}}(x,y) = \|\phi(x) - \phi(y)\| + \langle \psi(x) - \psi(y), \beta\rangle_+$.

And in order to learn an embedding that covers the relevant space, we define $\mu$ as the induced distribution by sampling $(x_0, x_1) \sim \pi$ and returning the element of the tangent bundle $(x_t, \dot{x}_t)$ for $t \sim U[0,1]$.  This training scheme is summarized in Algorithm~\ref{alg:training}.

\begin{algorithm}[t]
\raggedright
\caption{Finsler Metric Training}
\label{alg:training}
\begin{algorithmic}[1]
\Require Datasets $\mathcal{D}^0 = \{(x_i^0, y_i^0)\}$, $\mathcal{D}^1 = \{(x_i^1, y_i^1)\}$, pretrained Riemannian metric $\|\cdot\|_g$
\Ensure Finsler networks $(\phi, \psi, \beta)$, geodesic network $\eta$

\While{Training}
    \State Sample mini-batch $\mathcal{B} \subset \mathcal{D}^0 \cup \mathcal{D}^1$
    \State Compute $\mathcal{L}(f) = E_{(x,y)\sim B} CE(f(x), y)$
    \State OptimizerStep($f$, $\mathcal{L}$)
\EndWhile
\While{Training}
    \State Sample batches $\mathcal{B}^0 \subset \mathcal{D}^0$ and $\mathcal{B}^1 \subset \mathcal{D}^1$
    \State Compute $\pi = OT(\mathcal{B}^0, \mathcal{B}^1, d_{\hat{F}})$
    \State Compute $\mu = (x_t, \dot{x}_t)_\#(\pi \times U[0,1])$
    \State Compute $\mathcal{L} = L_{\mathrm{emb}}(\phi,\psi,\beta,\mu) + L_{geo}(\eta, \pi)$
    \State OptimizerStep($\phi$, $\psi$, $\beta$, $\eta$, $\mathcal{L}$)
\EndWhile

\State \Return $\phi$, $\psi$, $\beta$, $\eta$
\end{algorithmic}
\end{algorithm}

\section{Results}

We evaluate the proposed lineage-informed Finsler geometry in  three data settings: synthetic, zebrafish embryogenesis, and mouse organogenesis.  In all cases, we train on a single pair of timepoints, and interpolate held-out times. We measure accuracy using Wasserstein distance $W_1$.  We prioritize methods that support learning gedoesics under a geometry-informed metric to study how the inclusion of a discrete Finsler metric can improve performance.

\subsection{Synthetic Data}

We first validate the method on 2-dimensional synthetic data with simple lineage trees.  These instances demonstrate how the geodesics change when using Finsler geometry with a trained classifier, see Figure~\ref{fig:synthetic}.  We consider a setting where the geometric prior is completely uninformative as for the observed time points, the data exhibits vertical symmetry.  It is instead necessary to have the discrete prior and cluster labels define a metric that prefers the correct data cluster and breaks the symmetry. We additionally increase the number of ambient dimensions to mimic a high-dimensional setting to confirm there is no geometric collapse in our formulation. As shown in Table \ref{tab::syn} and Figure~\ref{fig:synthetic}, adding the Finsler geometry correctly curves the trajectory to follow the geometry of the ground truth lineage tree($0 \rightarrow 1 \rightarrow 4$) while CFM trajectory would follow the $0 \rightarrow 2 \rightarrow 4$ straight-line.  We observe a similar phenomenon in the 50 dimensional regime.  The improvement in Wasserstein distance is still substantial, but smaller because each data cloud is based on a high dimensional isotropic Gaussian and the curse of dimensionality causes Wasserstein distances to concentrate more slowly.  Nevertheless, we see a similar phenomenon where the Finsler guided trajectories have the correct lineage even in higher dimension.

\begin{figure*}
     \centering
     \begin{subfigure}[b]{0.3\textwidth}
         \centering
         \includegraphics[width=\textwidth]{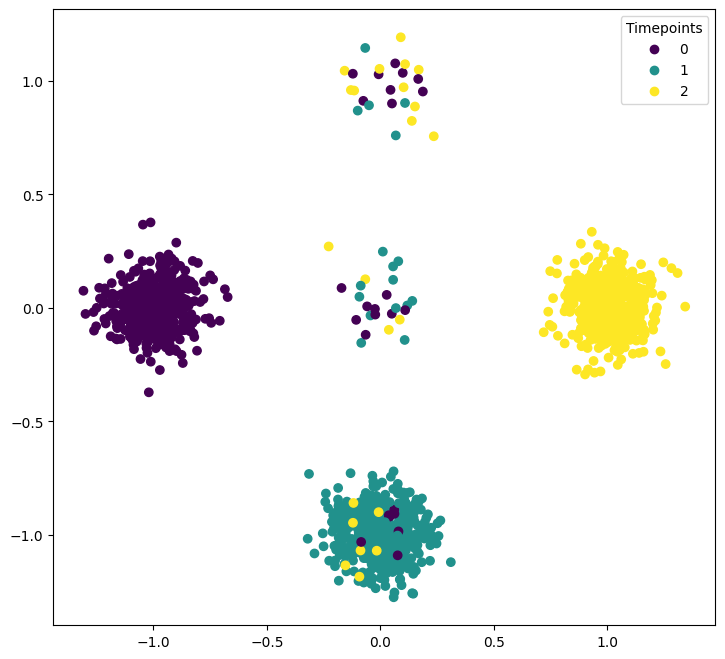}
         \caption{Data labeled by time}
         \label{fig:syn_time}
     \end{subfigure}
     \begin{subfigure}[b]{0.3\textwidth}
         \centering
         \includegraphics[width=\textwidth]{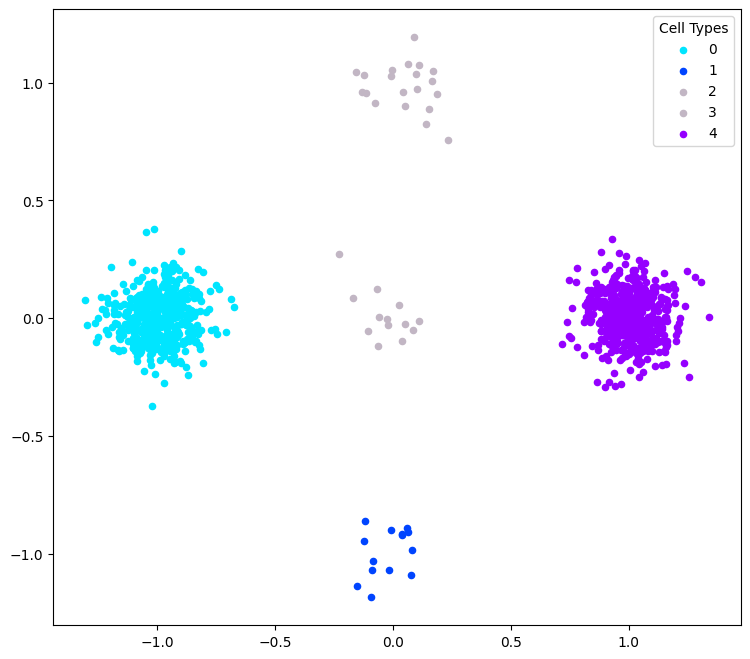}
         \caption{Observed data labeled by cell type}
         \label{fig:syn_cell}
     \end{subfigure}
     \begin{subfigure}[b]{0.3\textwidth}
         \centering
         \includegraphics[width=\textwidth]{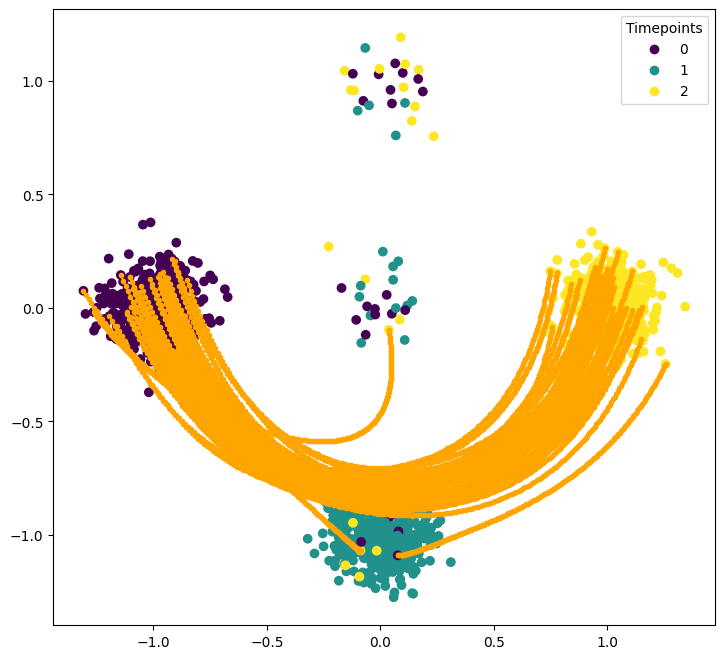}
         \caption{Sample of 50 geodesic trajectories}
         \label{fig:syn_traj}
     \end{subfigure}
        \caption{The Finsler metric trained on synthetic 2d data.  The true lineage goes from class $0 \rightarrow 1 \rightarrow 4$, but the model is only trained on $t \in \{0, 2\}$ and therefore must rely on classification guidance to bias the trajectories through the dark blue cell type.}
        \label{fig:synthetic}
\end{figure*}

\begin{table}[h]
\centering
\resizebox{\linewidth}{!}{%
\begin{tabular}{c c c}
\hline
 & Synthetic($n$=2) & Synthetic($n$=50) \\
\hline
CFM(\cite{tong2024improving}) & $0.95 \pm 0.00$ & $1.21 \pm 0.00$ \\
CFM+Finsler                  & $\bm{0.43 \pm 0.02}$ & $\bm{0.97 \pm 0.03}$ \\
\hline
\end{tabular}%
}
\caption{Wasserstein-1 distance averaged over marginals on left-out time points for the synthetic dataset. $n$ represents feature dimension.}
\label{tab::syn}
\end{table}

\subsection{Zebrafish Embryo Development}
Considering the diverse lineages of biological systems, we choose to subset cells in the zebrafish atlas ~\citet{saunders2023embryo} based on tissue origin, to clearly examine the expected lineage relationship between cell types. This atlas densely profiles embryogenesis through early larval stages. We first restrict our attention to brain tissue (central nervous system, CNS), the tissue with the most dynamic lineage tree available, training trajectories starting at 24hpf (hours post-fertilization) and ending at 72hpf. In this system, neurogenesis proceeds from proliferative neural progenitors toward progressively fate-restricted neuronal and glial populations, and the curated lineage graph provided by ~\citet{saunders2023embryo} encodes these directed developmental relationships based on integration of temporal progression, transcriptional programs, and prior embryological knowledge. Importantly, these edges reflect biologically supported differentiation trajectories rather than purely transcriptomic similarity.

From a developmental biology perspective, the CNS provides a stringent test case: neural progenitors at 24 hpf are transcriptionally plastic and multipotent, whereas by 72 hpf many lineages have committed to terminal or near-terminal neuronal identities. Thus, any interpolated trajectory between these stages should respect the hierarchical restriction of fate potential and avoid biologically implausible “backward” transitions (e.g., dedifferentiation).

\begin{table}[h]
\centering
\resizebox{\linewidth}{!}{%
\begin{tabular}{c c c}
\hline
 & Zebrafish-CNS & Zebrafish-PA \\
\hline
LineageOT(\cite{forrow2021lineageot}) & $15.01 \pm 0.00$ & $\bm{11.68 \pm 0.00}$\\
Moslin(\cite{lange2024mapping}) & $14.67 \pm 0.00$ & $12.15 \pm 0.00$\\

$\text{SF}^2 \text{M}$(\cite{tong2024simulation}) & $15.91 \pm 0.11$ & $13.42 \pm 0.10$\\
CFM(\cite{tong2024improving})                    & $14.69 \pm 0.02$ & $12.16 \pm 0.02$\\
CFM+Finsler                                      & $\bm{14.57 \pm 0.02}$ & $11.99 \pm 0.04$\\
\hline
MFM(\cite{kapusniak2024metric})                  & $13.85 \pm 0.03$ & $\bm{11.42 \pm 0.03}$ \\
MFM + Finsler                                    & $\bm{13.78 \pm 0.03}$ & $\bm{11.42 \pm 0.02}$ \\
\hline
\end{tabular}%
}
\caption{Wasserstein-1 distance averaged over marginals on left-out time points for 100-PC representation of the zebrafish dataset.}
\label{tab:zebrafish}
\end{table}

We then focus on the Pharyngeal Arch (PA) tissue where the lineage tree has a unique structure in~\citet{saunders2023embryo} reflecting cranial neural crest cells and mesodermal derivatives that give rise to skeletal, connective, and muscle cell types. We trained trajectories with the same starting and ending timepoints as in brain tissues. For both datasets, we benchmark our method against other models with learned metrics and geodesics by interpolating two held-out time points as measured in $W_1$ distance. As shown in Table~\ref{tab:zebrafish}, for both CFM and MFM, adding Finsler geometry of the lineage tree improved interpolation performance in both datasets. Additionally, by visualizing the learned trajectories via a river plot, we observe that the Finsler metric compelled learned trajectories to follow the lineage tree (see Figure~\ref{fig:zebrafish-two-panel}). For example, sampled trajectories starting from neuronal progenitors are being curved to descendant neuronal subtypes that are reachable under the curated developmental graph, with reduced probability mass assigned to unrelated branches (e.g., non-neural derivatives). This behavior is aligned with the process of CNS development, where progenitor populations undergo  transcriptional specialization driven by regulatory cascades (e.g., proneural factors followed by subtype-specific transcription factors), and intermediate states form a continuum rather than abrupt jumps. The Finsler-constrained model captures branch-consistent flows that traverse plausible intermediate transcriptional states rather than shortcutting across transcriptionally similar but developmentally unrelated clusters.

\captionsetup[subfigure]{font=small}

\begin{figure*}[t]
  \centering

  \begin{subfigure}[t]{0.495\textwidth}
    \centering
    \includegraphics[width=\textwidth,trim=8 8 8 8,clip]{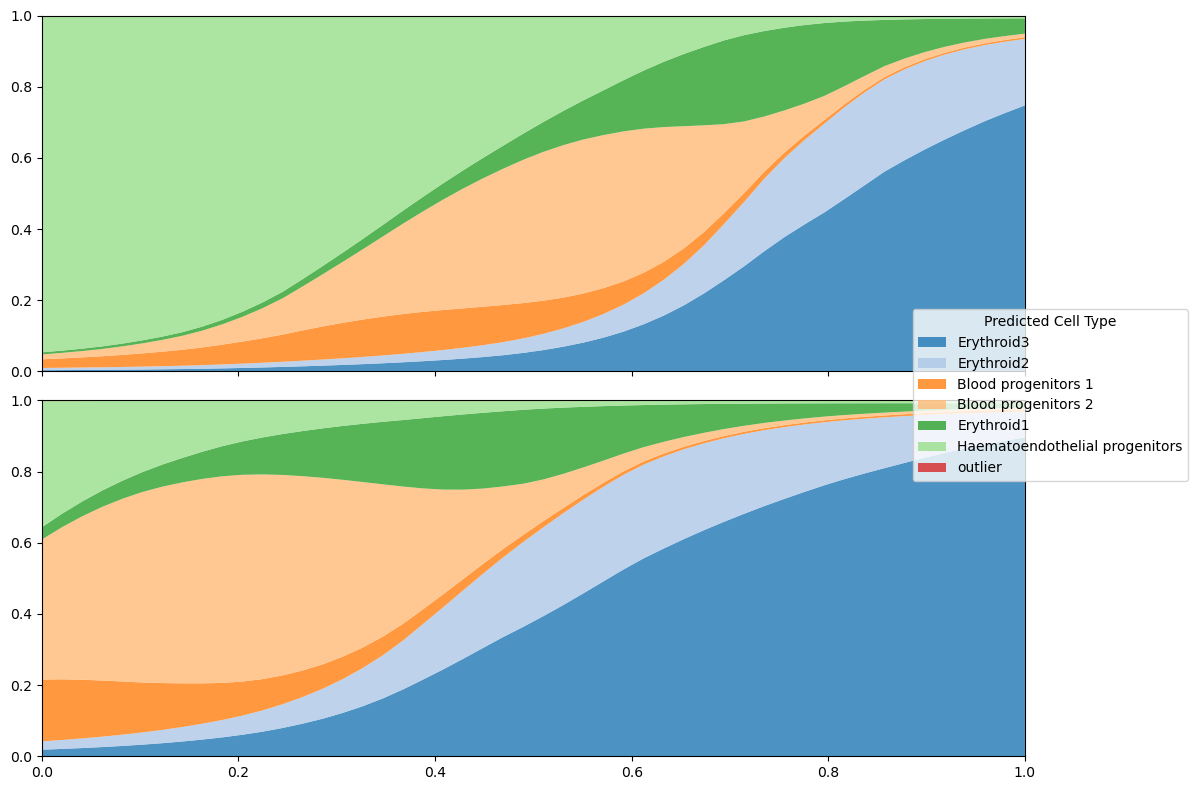}
    \caption{Finsler geometry.}
    \label{fig:mouse-river}
  \end{subfigure}\hfill
  \begin{subfigure}[t]{0.495\textwidth}
    \centering
    \includegraphics[width=\textwidth,trim=8 8 8 8,clip]{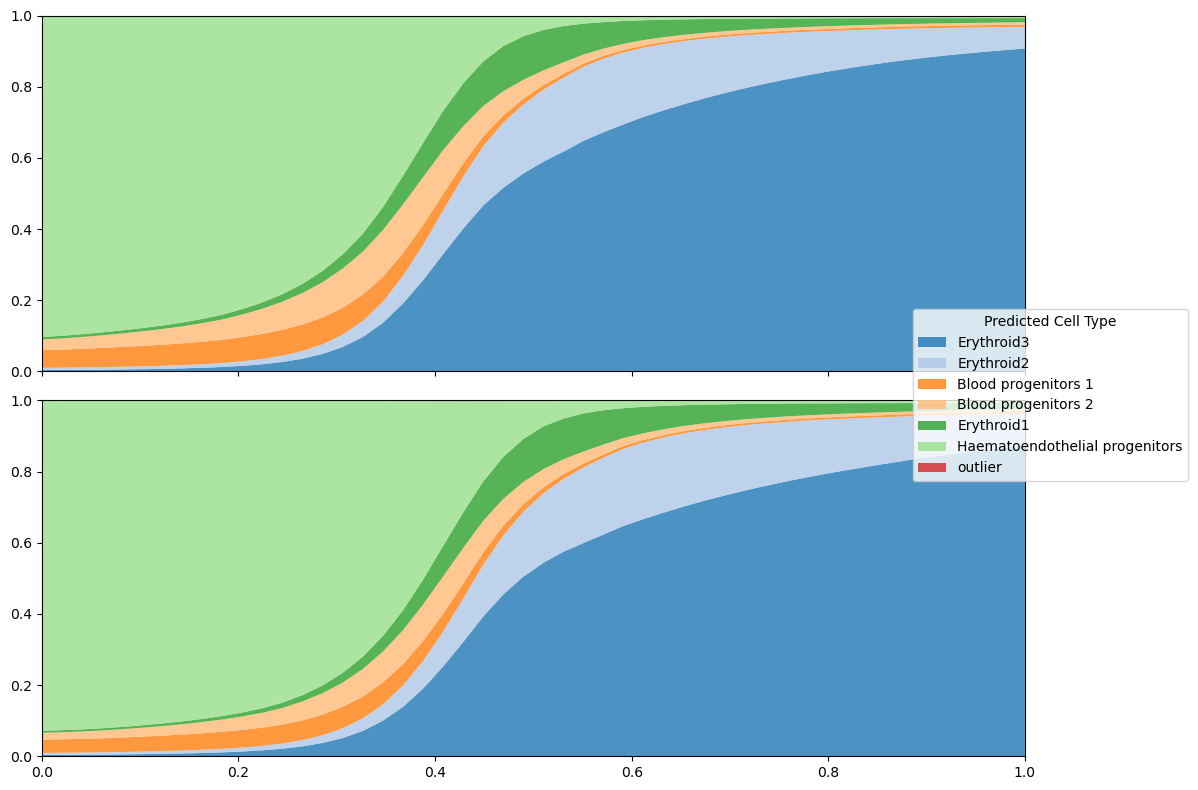}
    \caption{CFM only.}
    \label{fig:mouse-river-cfm}
  \end{subfigure}

  \caption{Cell type predictions along two learned trajectories for the Mouse-Blood dataset. The x-axis is time and the y-axis is cell type probability. Top and bottom are two sampled trajectories.}
  \label{fig:mouse-blood-two-panel}
\end{figure*}

\subsection{Mouse Embryo Organogenesis}

We then shift our focus onto a more complicated model organism and evaluate our model's performance on mouse embryo organogenesis data. The data is measured in embryonic days (E) at 8 time points between E6.5 and E8.5.  This period isolates key states of organogenesis~\citet{pijuan2019single} and is therefore the most applicable for lineage tree priors as there is rapid cell differentiation, but limited sample sizes for specific tissues.

In this setting, we first consider hematopoiesis. In the UMAP we observe a topologically linear trajectory from Hematoendothelial progenitors to the last measured stage of Erythroid cells. In particular, intermediate progenitor states represent transient programs that should be traversed, when interpolating between early progenitors and late erythroid populations.  However, the trajectory is not necessarily linear in the ambient space, and sparsity or noise in the data can make the classification prior a useful supplement to the geometric signal.

\begin{table}[h]
\centering
\resizebox{\linewidth}{!}{%
\begin{tabular}{c c}
\hline
 & Mouse-Blood \\
\hline
LineageOT(\cite{forrow2021lineageot}) & $7.98 \pm 0.00$\\
Moslin(\cite{lange2024mapping}) & $14.67 \pm 0.00$\\
$\text{SF}^2 \text{M}$(\cite{tong2024simulation}) & $9.16 \pm 0.01$\\
CFM(\cite{tong2024improving})                    & $8.07 \pm 0.08$\\
CFM+Finsler                                      & $\bm{7.65 \pm 0.04}$\\
\hline
MFM(\cite{kapusniak2024metric})                  & $7.95 \pm 0.08$\\
MFM + Finsler                                    & $\bm{7.65 \pm 0.06}$\\
\hline
\end{tabular}%
}
\caption{Wasserstein-1 distance averaged over marginals on left-out time points for 100-PC representation of the mouse embryo organogenesis dataset.}
\label{tab:mouse}
\end{table}

For our tissue-specific dataset, we train trajectories starting at 7.0E and ending at 8.5E. The interpolation performance are measured via mean $W_1$ distance over the 5 held-out time points. From Table~\ref{tab:mouse}, we see that adding Finsler geometry again improves both CFM and MFM on both datasets. 
The consistent improvement across model classes suggests that the directed prior is successfully ruling out biologically implausible transport under noisy or sparse sampling rather than redefining the dominant geometry. 
We show that learned trajectories follow the lineage tree structure in Figure~\ref{fig:mouse-river}. For example, trajectories initialized from cells with both hematoendothelial progenitors and blood progenitor 2 states flow to cells with Erythroid 2 and Erythroid 3 states at the end, which matches the structure of the tree and known developmental biology~\citet{pijuan2019single}. Biologically, the Finsler term encourages geodesics to allocate probability mass through fate-restricted intermediates (e.g., blood progenitor programs) before reaching terminal erythroid states, aligning interpolation with the expected ordering of hematopoietic maturation over E7.0–E8.5.
Meanwhile, we also show an example of CFM failing to capture the correct trajectory and lineage on the same dataset in Figure~\ref{fig:mouse-river-cfm}, where the trajectories starting at hematoendothelial progenitors tend to skip the intermediate states and go straight to the endpoint state which is Erythroid 3.

\section{Discussion}

We have introduced a lineage-informed Finsler geometry for trajectory inference that combines a continuous geometric prior
with a discrete, directed biological prior encoded by a lineage adjacency matrix $A$. The key modeling contribution is an asymmetric local norm: the classifier-induced term penalizes instantaneous motion that transfers probability mass into lineage-inconsistent directions, so geodesics become direction-dependent and better reflect developmental transitions. To counter the necessity of oracle geodesics, we also show that Finsler geometry can be learned by learning isometric embeddings optimized for distance preservation and energy minimization.

Empirically, adding the Finsler geometry to existing geodesic based models improves interpolation on both synthetic and real single-cell datasets. On synthetic examples it also resolves geometric ambiguities by selecting lineage-consistent paths, while on zebrafish and mouse tissues it yields consistent gains in held-out $W_1$, indicating that lineage
consistency is a useful inductive bias under noise and sparse sampling. Moreover, our Finsler metric is a differentiable, local adjustment of
path cost, and is therefore compatible with OT-based couplings and end-to-end geodesic learning at scale.

Limitations of the method include dependence on the quality of the lineage prior, reduced benefit when the lineage structure is too dense or already aligned with the ambient geometry, and sensitivity to classifier calibration and embedding/geodesic network quality.

\section{Conclusion}
We introduced a lineage-informed Finsler metric for trajectory inference that unifies continuous geometric priors with
directed, discrete lineage constraints via an asymmetric local norm derived from a classifier and a lineage adjacency
matrix. This construction yields direction-aware geodesics that better reflect developmental dynamics and
integrates naturally with OT-based coupling and flow/trajectory learning objectives. Across synthetic benchmarks and
real single-cell developmental datasets, incorporating the proposed Finsler geometry into CFM/MFM improves held-out
timepoint interpolation as measured by marginal $W_1$, and qualitatively produces trajectories aligned with lineage
structure. Future work includes improving robustness to lineage mis-specification, reducing sensitivity to classifier
calibration, and extending beyond tree priors to more general directed fate graphs and uncertainty-aware priors.



\begin{acknowledgements} 
This work was made possible by support from the MacMillan Family and the MacMillan Center for the Study of the Non-Coding Cancer Genome at the New York Genome Center. A.Z. is the Sijbrandij Foundation Quantitative Biology Fellow of the Damon Runyon Cancer Research Foundation (DRQ26-25). This work was also supported by the NIH NHGRI grant R01HG012875, and grant number 2022-253560 from the Chan Zuckerberg Initiative DAF, an advised fund of Silicon Valley Community Foundation.
\end{acknowledgements}

\bibliography{uai2026-template}

\newpage

\onecolumn

\title{Learning lineage-guided geodesics with Finsler Geometry\\(Supplementary Material)}
\maketitle

\appendix
\section{Experimental Details}

\subsection{Dataset Details}

For the synthetic data, we sample Gaussians with 0.1 standard deviation at 5 locations, with 500 points for the start and end point, and 25 points at each of the 3 intermediate clusters with distinct cell types.

All single cell data follows standard preprocessing in Scanpy~\citep{wolf2018scanpy} of count normalizaiton, applying log transform, and subsetting to the top 2000 highly variable genes, followed by PCA to reduce to dimension 100. 

We consider how we derive lineage trees.  For the Zebrafish-CNS dataset, we use the Zebrafish brain lineage tree derived in Extended Data Fig. 2 of~\citet{saunders2023embryo}.  For Zebrafish-PA, we apply the PAGA algorithm~\citep{wolf2019paga} with threshold 0.2 to the PC reduced data, and prune cell types that are isolated in the tree.  For the Mouse blood dataset we select directly the lineage to be a path from Hematoendothelial progenitors to Erythroid 3 as apparent on the UMAP in Figure 1 in~\citet{pijuan2019single}, and verify with PAGA the same tree, restricting to these cell types.

\subsection{Parameterization Details}

We parameterize all networks with 3 layers of 256 hidden dimension, with the exception of the latent dimension of the distance embeddings $l = 100$, using batchnorm and trained with AdamW~\citep{loshchilov2017decoupled} using learning rate 0.001.  We use batch size 512 for classification and metric training, and 2048 for embedding and geodesic training.  The geodesic network embeds time with a sinusoidal embedding on 32 frequencies.

We use the implementation of MFM~\citep{kapusniak2024metric} with the same networks as our model for consistency.  For $\text{SF}^2\text{M}$, we adopt the implementation of~\cite{tong2024simulation}—specifically, the exact configuration provided in the released GitHub repository (\url{https://github.com/atong01/conditional-flow-matching})—and modify only the time embedding where we use sinusoidal embeddings to be consistent with other models.

For validation, we take one sample from the first heldout timepoint in each experiment and calculate the W1 against that sample only, then based on the optimal hyperparameters for validation test on all other samples at all timepoints.  We use 10 independent random seeds and report mean and standard variation of test error.

\begin{table}[h]
 \caption{Hyperparameter ranges for Finsler metric parameters}
 \label{tab:hyperparams-fin}
\centering
 \begin{tabular}{cc} 
 \hline
 Hyperparameter & Range \\
  \hline
  $\lambda$ & $\{0.2, 0.5, 1.0\}$ \\
  classifier smoothing & $\{0.03, 0.05\}$ \\
  \hline
 \end{tabular}
\end{table}

\begin{table}[h]
 \caption{Hyperparameter ranges for Metric Flow Matching specific parameters}
 \label{tab:hyperparams-mfm}
\centering
 \begin{tabular}{cc} 
 \hline
 Hyperparameter & Range \\
  \hline
  $\#$ clusters $r$ & $\{100, 150, 200\}$ \\
  kernel bandwidth $\kappa$ & $\{1, 1.5\}$ \\
  kernel smoothing parameter $\epsilon$ & $\{0.05, 0.1, 0.2\}$ \\
  use Euclidean optimal transport & $\{True, False\}$ \\
  \hline
 \end{tabular}
\end{table}

\captionsetup[subfigure]{font=small}

\begin{figure*}[t]
  \centering

  \begin{subfigure}[t]{0.6\textwidth}
    \centering
    \includegraphics[width=\textwidth,trim=8 8 8 8,clip]{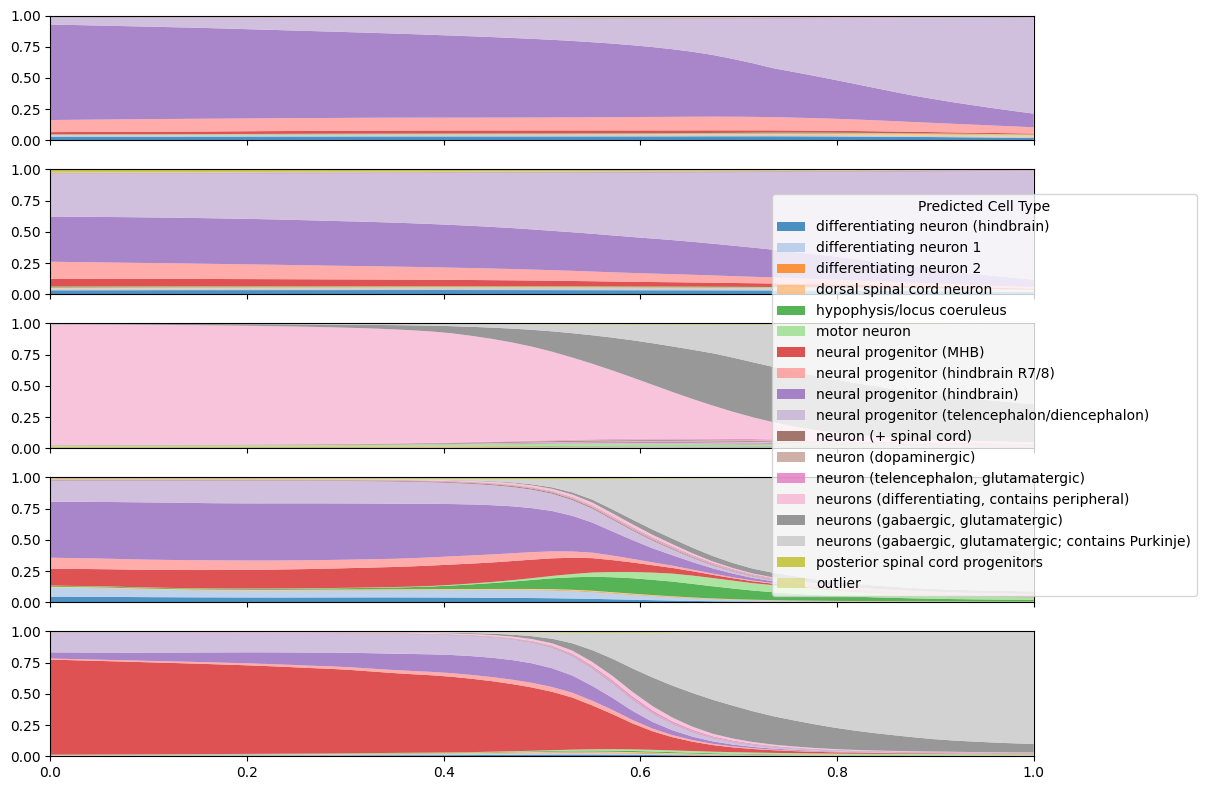}
    \caption{Finsler geometry.}
    \label{fig:zebrafish-river}
  \end{subfigure}
  
  \begin{subfigure}[t]{0.6\textwidth}
    \centering
    \includegraphics[width=\textwidth,trim=8 8 8 8,clip]{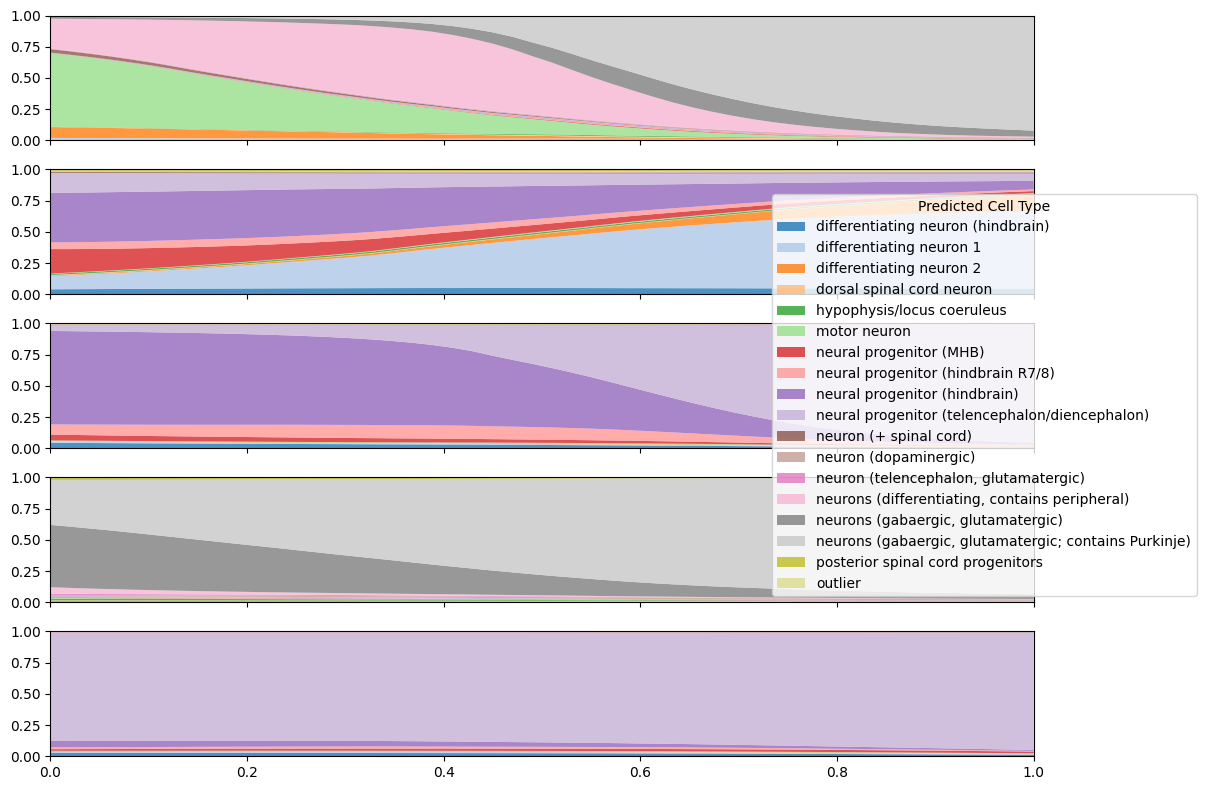}
    \caption{CFM only.}
    \label{fig:zebrafish-river-cfm}
  \end{subfigure}

  \caption{Cell type predictions along five learned trajectories for the Zebrafish-CNS dataset. The x-axis is time and the y-axis is cell type probability. Top and bottom are five sampled trajectories.}
  \label{fig:zebrafish-two-panel}
\end{figure*}

\begin{figure*}[t]
  \centering

  \begin{subfigure}[t]{0.6\textwidth}
    \centering
    \includegraphics[width=\textwidth,trim=8 8 8 8,clip]{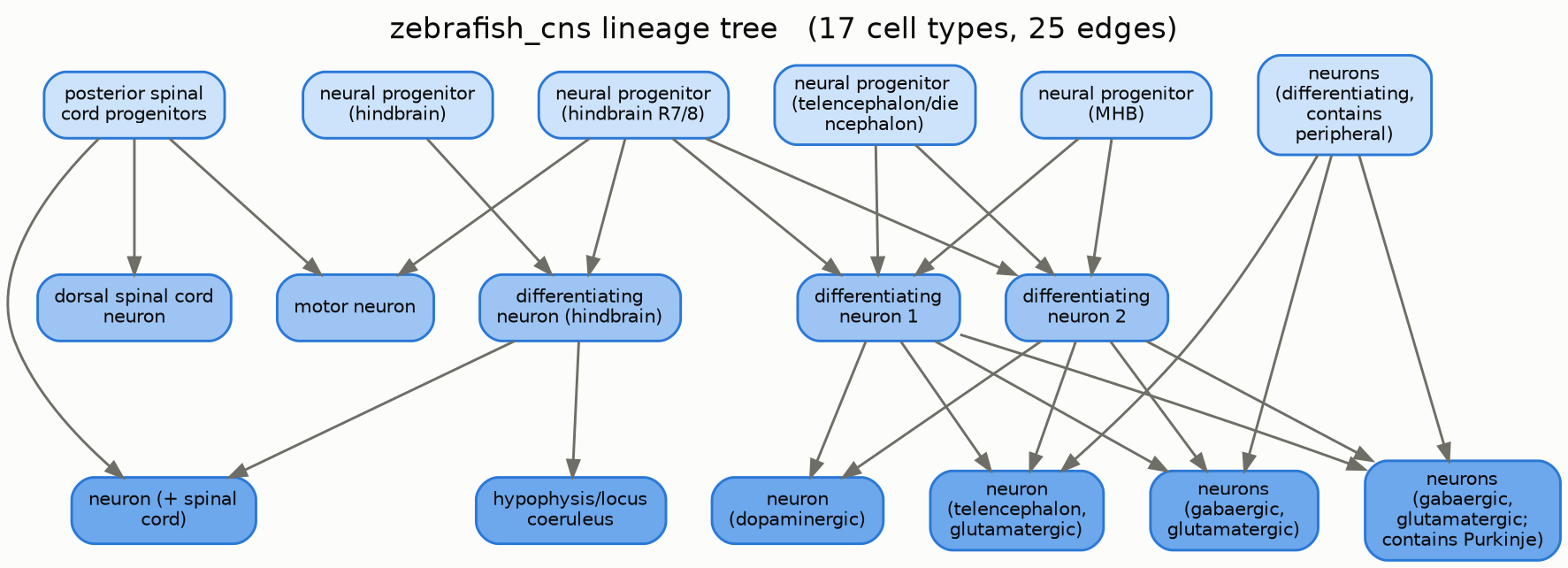}
    \caption{Lineage tree for Zebrafish Central Nervous System}
    \label{fig:lin-zebrafish-cns}
  \end{subfigure}
  
  \begin{subfigure}[t]{0.6\textwidth}
    \centering
    \includegraphics[width=\textwidth,trim=8 8 8 8,clip]{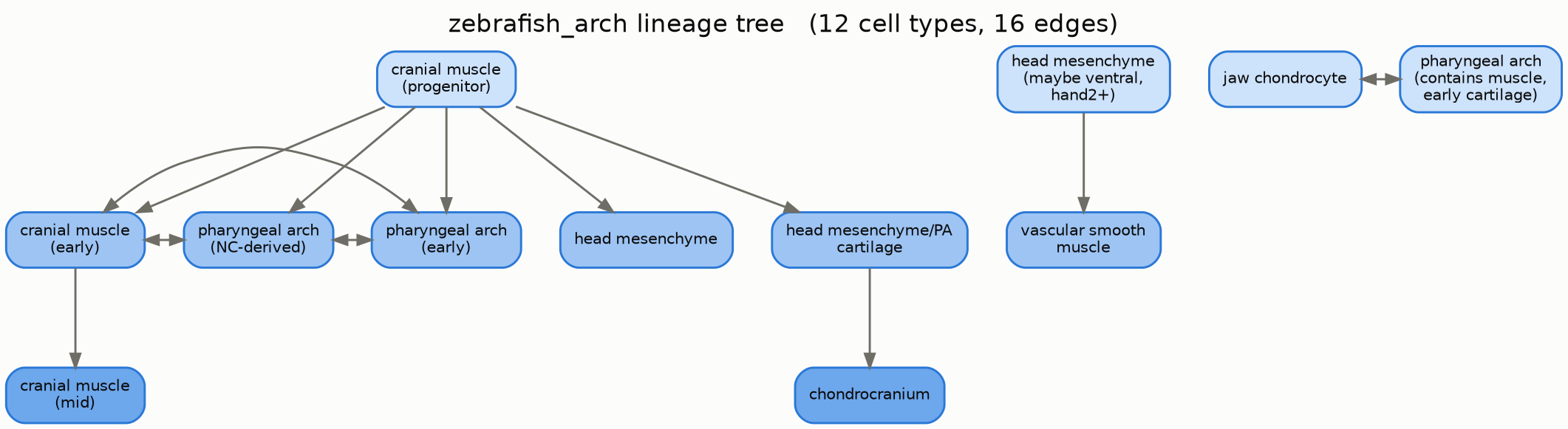}
    \caption{Lineage tree for Zebrafish Pharyngeal Arch}
    \label{fig:lin-zebrafish-arch}
  \end{subfigure}

  \begin{subfigure}[t]{0.3\textwidth}
    \centering
    \includegraphics[width=\textwidth,trim=8 8 8 8,clip]{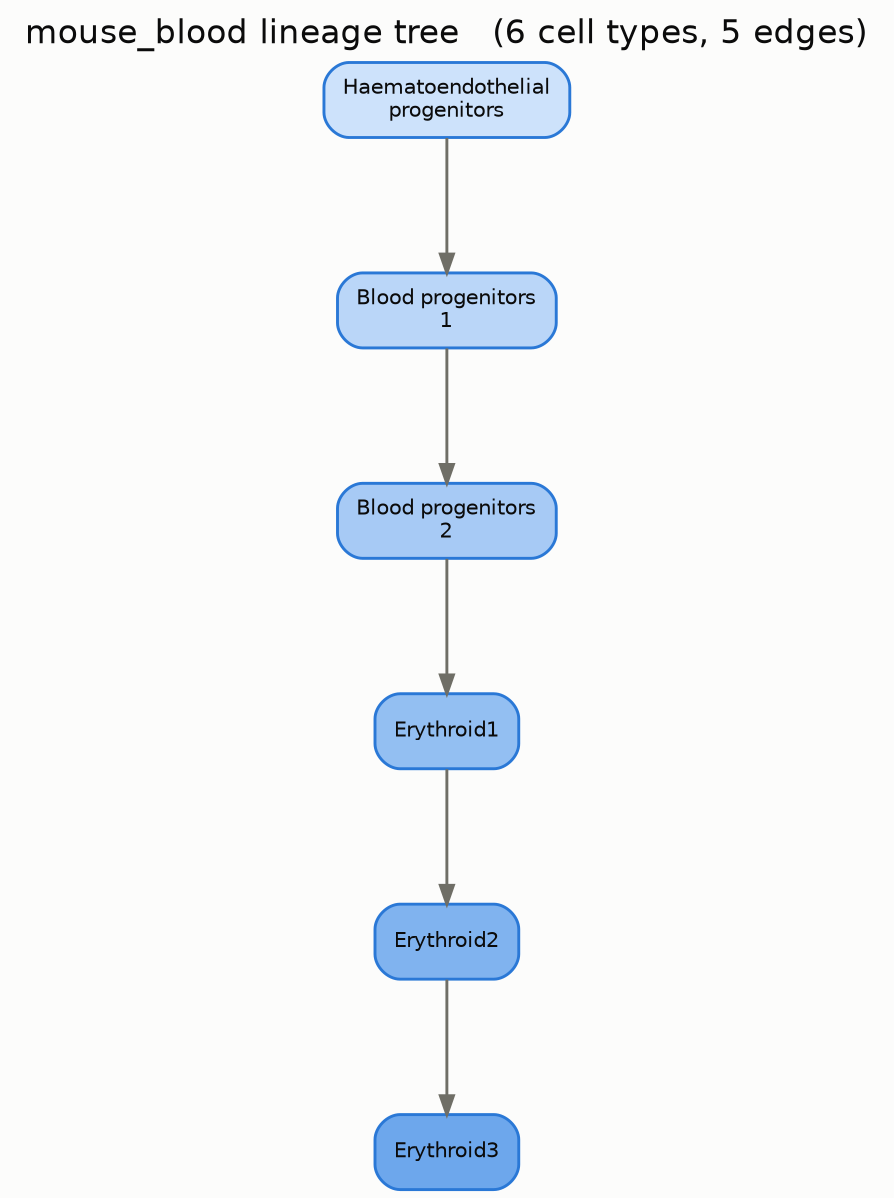}
    \caption{Lineage tree for Mouse Blood}
    \label{fig:lin-mouse}
  \end{subfigure}

  \caption{Lineage tree priors used in all single-cell experiments}
  \label{fig:lineage}
\end{figure*}

\end{document}